\newif\ifdraft \drafttrue
\newif\iffull \fulltrue

\fulltrue\draftfalse
\documentclass[11pt]{article}
\usepackage{url}  
\usepackage{graphicx}  
\usepackage{subfigure}
\usepackage{algorithm}
\usepackage{algorithmic}
\usepackage{natbib}
\usepackage{fullpage}
\usepackage{amssymb,amsmath,amsthm}
\newtheorem{claim}{Claim}

\newtheorem{lemma}{Lemma}
\newtheorem{theorem}{Theorem}

\newtheorem{corollary}{Corollary}

\usepackage{xcolor}
\usepackage{subfig}
\theoremstyle{definition}
\newtheorem{definition}{Definition}
\theoremstyle{remark}
\newtheorem{remark}{Remark}

\usepackage{bbm}

\newcommand{\inner}[2]{\langle #1, #2 \rangle}

\newcommand{\xit}{\ensuremath{x_{it}}}
\newcommand{\yit}{\ensuremath{y_{i,t}}}

\newcommand{\D}{\mathcal{D}}

\newcommand{\cH}{\mathcal{H}}

\newcommand{\E}[1]{{\mathbb{E}\left[#1 \right]}}

\newcommand{\Ex}[2]{\mathbb{E}_{#1}\left[ #2 \right]}

\renewcommand{\Pr}[1]{\mathbb{P}\left[#1\right]}
\newcommand{\Prob}[2]{\mathbb{P}_{#1}\left[#2\right]}

\newcommand{\A}{\mathcal{A}}

\newcommand{\bE}{\mathop{\mathbb{E}}}
\newcommand{\bias}{\mathrm{Bias}}
\newcommand{\cX}{\mathcal{X}}
\newcommand{\eps}{\epsilon}
\title{Mitigating Bias in Adaptive Data Gathering via Differential Privacy}
\author{Seth Neel\thanks{Department of Statistics, The Wharton School, University of Pennsylvania. \texttt{sethneel@wharton.upenn.edu.} Supported in part by a 2017 NSF Graduate Research Fellowship.}
 \and Aaron Roth\thanks{Department of Computer and Information Sciences, University of Pennsylvania. \texttt{aaroth@cis.upenn.edu.} Supported in part by grants from the DARPA Brandeis project, the Sloan Foundation, and NSF grants CNS-1513694 and CNS-1253345.}}
\begin{document}
\maketitle

\begin{abstract}
Data that is gathered adaptively --- via bandit algorithms, for example --- exhibits bias. This is true both when gathering simple numeric valued data --- the empirical means kept track of by stochastic bandit algorithms are biased downwards --- and when gathering more complicated data --- running hypothesis tests on complex data gathered via contextual bandit algorithms leads to false discovery. In this paper, we show that this problem is mitigated if the data collection procedure is differentially private. This lets us both bound the bias of simple numeric valued quantities (like the empirical means of stochastic bandit algorithms), and correct the $p$-values of hypothesis tests run on the adaptively gathered data. Moreover, there exist differentially private bandit algorithms with near optimal regret bounds: we apply existing theorems in the simple stochastic case, and give a new analysis for linear contextual bandits. We complement our theoretical results with experiments validating our theory.
\end{abstract}

\thispagestyle{empty} \setcounter{page}{0}
\clearpage

\section{Introduction}
\label{intro}
Many modern data sets consist of data that is gathered \emph{adaptively}: the choice of whether to collect more data points of a given type depends on the data already collected. For example, it is common in industry to conduct ``A/B'' tests to make decisions about many things, including ad targeting, user interface design, and algorithmic modifications, and this A/B testing is often conducted using ``bandit learning algorithms'' \cite{bandits}, which adaptively select treatments to show to users in an effort to find the best treatment as quickly as possible. Similarly, sequential clinical trials may halt or re-allocate certain treatment groups due to preliminary results, and empirical scientists may initially try and test multiple hypotheses and multiple treatments, but then decide to gather more data in support of certain hypotheses and not others, based on the results of preliminary statistical tests.

Unfortunately, as demonstrated by \cite{gather}, the data that results from adaptive data gathering procedures will often exhibit substantial \emph{bias}. As a result, subsequent analyses that are conducted on the data gathered by adaptive procedures will be prone to error, unless the bias is explicitly taken into account. This can be difficult. \cite{gather} give a selective inference approach: in simple stochastic bandit settings, if the data was gathered by a specific stochastic algorithm that they design, they give an MCMC based procedure to perform maximum likelihood estimation to recover de-biased estimates of the underlying distribution means.
In this paper, we give a related, but orthogonal approach whose simplicity allows for a substantial generalization beyond the simple stochastic bandits setting. We show that in very general settings, if the data is gathered by a differentially private procedure, then we can place strong bounds on the bias of the data gathered, without needing any additional de-biasing procedure. Via elementary techniques, this connection implies the existence of simple stochastic bandit algorithms with nearly optimal worst-case regret bounds, with very strong bias guarantees. The connection also allows us to derive algorithms for linear contextual bandits with nearly optimal regret guarantees, and strong bias guarantees. Since our connection to differential privacy only requires that the \emph{rewards} and not the \emph{contexts} be kept private, we are able to obtain improved accuracy compared to past approaches to private contextual bandit problems. By leveraging existing connections between differential privacy and adaptive data analysis \cite{preserve,agstab,maxinformation2}, we can extend the generality of our approach to bound not just bias, but to correct for effects of adaptivity on arbitrary statistics of the gathered data. For example, we can obtain valid $p$-value corrections for hypothesis tests (like $t$-tests)  run on the adaptively collected data. Since the data being gathered will generally be useful for some as yet unspecified scientific analysis, rather than just for the narrow problem of mean estimation, our technique allows for substantially broader possibilities compared to past approaches. Experiments explore the bias incurred by conventional bandit algorithms, confirm the reduction in bias obtained by leveraging privacy, and show why correction for adaptivity is crucial to performing valid post-hoc hypothesis tests. In particular we show that for the fundamental primitive of conducting $t$-tests for regression coefficients, naively conducting tests on adaptively gathered data leads to incorrect inference.

\subsection{Our Results}
This paper has four main contributions:
\begin{enumerate}
\item Using elementary techniques, we provide explicit bounds on the bias of empirical arm means maintained by bandit algorithms in the simple stochastic setting that make their selection decisions as a differentially private function of their observations. Together with existing differentially private algorithms for stochastic bandit problems, this yields an algorithm that obtains an essentially optimal worst-case regret bound, and guarantees minimal bias (on the order of $O(1/\sqrt{K\cdot T})$) for the empirical mean maintained for every arm.
\item We then extend our results to the linear contextual bandit problem. We show that algorithms that make their decisions in a way that is differentially private in the observed reward of each arm (but which need not be differentially private in the context) have bounded bias (as measured by the difference between the predicted reward of each arm at each time step, compared to its true reward). We also derive a differentially private algorithm for the contextual bandit problem, and prove new bounds for it. Together with our bound on bias, this algorithm also obtains strong sublinear regret bounds, while having robust guarantees on bias.
\item We then make a general observation, relating adaptive data gathering to an adaptive analysis of a fixed dataset (in which the choice of which query to pose to the dataset is adaptive). This lets us apply the large existing literature connecting differential privacy to adaptive data analysis \cite{DFHPRR15science,preserve,agstab}. In particular, it lets us apply the max-information bounds of \cite{maxinformation1,maxinformation2} to our adaptive data gathering setting. This allows us to give much more general guarantees about the data collected by differentially private collection procedures, that extend well beyond bias. For example, it lets us correct the $p$-values for arbitrary hypothesis tests run on the gathered data.
\item Finally, we run a set of experiments that measure the bias incurred by the standard UCB algorithm in the stochastic bandit setting, contrast it with the low bias obtained by a private UCB algorithm, and show that there are settings of the privacy parameter that simultaneously can make bias statistically insignificant, while having competitive empirical regret with the non-private UCB algorithm. We also demonstrate in the linear contextual bandit setting how failing to correct for adaptivity can lead to false discovery when applying $t$-tests for non-zero regression coefficients on an adaptively gathered dataset.
\end{enumerate}

\subsection{Related Work}
This paper bridges two recent lines of work. Our starting point is two recent papers: \cite{Vil15} empirically demonstrate in the context of clinical trials that a variety of simple stochastic bandit algorithms produce biased sample mean estimates (Similar results have been empirically observed in the context of contextual bandits \cite{DAI17}). \cite{gather} prove that simple stochastic bandit algorithms that exhibit two natural properties (satisfied by most commonly used algorithms, including UCB and Thompson Sampling) result in empirical means that exhibit negative bias. They then propose a heuristic algorithm which computes a maximum likelihood estimator for the sample means from the empirical means gathered by a modified UCB algorithm which adds Gumbel noise to the decision statistics.  \cite{DMST17} propose a debiasing procedure for ordinary least-squares estimates computed from adaptively gathered data that trades off bias for variance, and prove a central limit theorem for their method. In contrast, the methods we propose in this paper are quite different. Rather than giving an ex-post debiasing procedure, we show that if the data were gathered in a differentially private manner, no debiasing is necessary. The strength of our method is both in its simplicity and generality: rather than proving theorems specific to particular estimators, we give methods to correct the $p$-values for \emph{arbitrary} hypothesis tests that might be run on the adaptively gathered data.

The second line of work is the recent literature on \emph{adaptive data analysis} \cite{preserve,maxinformation1,HU14,SU15,RZ16,WLF16,agstab,ladder,CLNRW16,FS17,FS17b} which draws a connection between differential privacy \cite{DMNS06} and generalization guarantees for adaptively chosen statistics. The adaptivity in this setting is dual to the setting we study in the present paper: In the adaptive data analysis literature, the dataset itself is fixed, and the goal is to find techniques that can mitigate bias due to the adaptive selection of analyses. In contrast, here, we study a setting in which the data gathering procedure is itself adaptive, and can lead to bias even for a fixed set of statistics of interest.  However, we show that adaptive data gathering can be re-cast as an adaptive data analysis procedure, and so the results from the adaptive data analysis literature can be ported over.

\section{Preliminaries}
\label{prelims}
\subsection{Simple Stochastic Bandit Problems}
In a simple stochastic bandit problem, there are $K$ unknown distributions $P_i$ over the unit interval [0,1], each with (unknown) mean $\mu_i$. Over a series of rounds $t \in \{1,\ldots,T\}$, an algorithm $\A$ chooses an arm $i_t \in [K]$, and observes a reward $y_{i_t,t} \sim P_{i_t}$. Given a sequence of choices $i_1,\ldots,i_T$, the pseudo-regret of an algorithm is defined to be:
$$\mathrm{Regret}((P_1,\ldots,P_K),i_1,\ldots,i_T) = T\cdot \max_i \mu_i - \sum_{t=1}^T \mu_{i_t}$$
We say that regret is bounded if we can put a bound on the quantity $\mathrm{Regret}((P_1,\ldots,P_K),i_1,\ldots,i_T)$ in the worst case over the choice of distributions $P_1,\ldots,P_K$, and with high probability or in expectation over the randomness of the algorithm and of the reward sampling.

As an algorithm $\A$ interacts with a bandit problem, it generates a \emph{history} $\Lambda$ , which records the sequence of actions taken and rewards observed thus far: $\Lambda_t = \{(i_\ell,y_{i_\ell,\ell})\}_{\ell=1}^{t-1}$.
We denote the space of histories of length $T$ by $\cH^T = ([K]\times \mathbb{R})^T$.

The definition of an algorithm $\A$ induces a sequence of $T$ (possibly randomized) selection functions $f_t:\cH^{t-1}\rightarrow [K]$, which map histories onto decisions of which arm to pull at each round.
\subsection{Contextual Bandit Problems}
In the contextual bandit problem, decisions are endowed with observable features. Our algorithmic results in this paper focus on the \emph{linear} contextual bandit problem, but our general connection between adaptive data gathering and differential privacy extends beyond the linear case. For simplicity of exposition, we specialize to the linear case here.

There are $K$ arms $i$, each of which is associated with an unknown $d$-dimensional linear function represented by a vector of coefficients $\theta_i \in \mathbb{R}^d$ with $||\theta_i||_2 \leq 1$. In rounds $t \in \{1,\ldots,T\}$, the algorithm is presented with a \emph{context} $x_{i,t} \in \mathbb{R}^d$ for each arm $i$ with $||x_{i,t}||_2 \leq 1$, which may be selected by an adaptive adversary as a function of the past history of play. We write $x_t$ to denote the set of all $K$ contexts present at round $t$. As a function of these contexts, the algorithm then selects an arm $i_t$, and observes a reward $y_{i_t,t}$. The rewards satisfy $\E{y_{i,t}} = \theta_{i}\cdot x_{i,t}$ and are bounded to lie in $[0,1]$. Regret is now measured with respect to the optimal policy. Given a sequence of contexts $x_1,\ldots,x_t$, a set of linear functions $\theta_1,\ldots,\theta_k$, and a set of choices $i_1,\ldots,i_k$,  the pseudo-regret of an algorithm is defined to be:
$$\mathrm{Regret}((\theta_1,\ldots,\theta_K),(x_1,i_1),\ldots,(x_t,i_T)) = \sum_{t=1}^T \left(\max_i \theta_i \cdot x_{i,t} - \theta_{i,t}\cdot x_{i_t,t}\right)$$
We say that regret is bounded if we can put a bound on the quantity $\mathrm{Regret}((\theta_1,\ldots,\theta_K),(x_1,i_1),\ldots,(x_T,i_T))$ in the worst case over the choice of linear functions $\theta_1,\ldots,\theta_K$ and contexts $x_1,\ldots,x_T$, and with high probability or in expectation over the randomness of the algorithm and of the rewards.

In the contextual setting, histories incorporate observed context information as well: $\Lambda_t = \{(i_\ell,x_\ell,y_{i_\ell,\ell})\}_{\ell=1}^{t-1}$.

Again, the definition of an algorithm $\A$ induces a sequence of $T$ (possibly randomized) selection functions $f_t:\cH^{t-1} \times \mathbb{R}^{d\times K}\rightarrow [K]$, which now maps both a history and a set of contexts at round $t$ to a choice of arm at round $t$.

\subsection{Data Gathering in the Query Model}\label{ada}
Above we've characterized a bandit algorithm $\A$ as \emph{gathering} data adaptively using a sequence of selection functions $f_t$, which map the observed history $\Lambda_t \in \cH^{t-1}$ to the index of the next arm pulled. In this model only after the arm is chosen is a reward drawn from the appropriate distribution. Then the history is updated, and the process repeats.

In this section, we observe that whether the reward is drawn after the arm is ``pulled," or in advance, is a distinction without a difference. We cast this same interaction into the setting where an analyst asks an adaptively chosen sequence of queries to a fixed dataset, representing the arm rewards. The process of running a bandit algorithm $\A$ up to time $T$ can be formalized as the adaptive selection of  $T$ queries against a single database of size $T$ - fixed in advance. The formalization consists of two steps:
\begin{itemize}
\item By the principle of deferred randomness, we view any simple stochastic bandit algorithm as operating in a setting in which $T$ i.i.d. samples from $\prod_{i=1}^K P_i$ (vectors of length $K$ representing the rewards for each of $K$ arms on each time step $t$) are drawn before the interaction begins. This is the \textbf{Interact} algorithm below.

In the contextual setting, the contexts are also available, and the $T$ draws are not drawn from identical distributions. Instead, the $t^{th}$ draw is from $\prod_{i=1}^K P_i^t$, where each distribution $P_i^t$ is determined by the context $x_i^t$.

\item The choice of arm pulled at time $t$ by the bandit algorithm can be viewed as the answer to an adaptively selected query against this fixed dataset.  This is the \textbf{InteractQuery} algorithm below.
\end{itemize}

Adaptive data analysis is formalized as an interaction in which a data analyst $\A$ performs computations on a dataset $D$, observes the results, and then may choose the identity of the next computation to run as a function of previously computed results \cite{preserve,DFHPRR15science}. A sequence of recent results shows that if the queries are differentially private in the dataset $D$, then they will not in general overfit $D$, in the sense that the distribution over results induced by computing $q(D)$ will be ``similar'' to the distribution over results induced if $q$ were run on a new dataset, freshly sampled from the same underlying distribution \cite{preserve,DFHPRR15science,agstab,maxinformation1,maxinformation2}. We will be more precise about what these results say in Section \ref{mutual}.

 Recall that histories $\Lambda$ record the choices of the algorithm, in addition to its observations. It will be helpful to introduce notation that separates out the choices of the algorithm from its observations. In the simple stochastic setting and the contextual setting, given a history $\Lambda_t$, an \emph{action history} $\Lambda_t^{\A} = (i_1,\ldots,i_{t-1}) \in [K]^{t-1}$ denotes the portion of the history recording the actions of the algorithm.

In the simple stochastic setting, a \emph{bandit tableau} is a $T\times K$ matrix $D \in \left([0,1]^K\right)^T$. Each row $D_t$ of $D$ is a vector of $K$ real numbers, intuitively representing the rewards that would be available to a bandit algorithm at round $t$ for each of the $K$ arms. In the contextual setting, a bandit tableau is represented by a pair of $T\times K$ matrices: $D \in \left([0,1]^K\right)^T$ and $C \in \left((\mathbb{R}^d)^K \right)^T$. Intuitively, $C$ represents the \emph{contexts} presented to a bandit algorithm $\A$ at each round: each row $C_t$ corresponds to a set of $K$ contexts, one for each arm. $D$ again represents the rewards that would be available to the bandit algorithm at round $t$ for each of the $K$ arms.

We write $\mathrm{Tab}$ to denote a bandit tableau when the setting has not been specified: implicitly, in the simple stochastic case, $\mathrm{Tab} = D$, and in the contextual case, $\mathrm{Tab} = (D, C)$.

Given a bandit tableau and a bandit algorithm $\A$, we have the following interaction:
\begin{algorithm}
\textbf{Interact}\newline
\textbf{Input}: Time horizon $T$, bandit algorithm $\A$, and bandit tableau $\mathrm{Tab}$ ($D$ in the simple stochastic case, $(D, C)$ in the contextual case).\newline
\textbf{Output}: Action history $\Lambda_{T+1}^\A \in [K]^T$
\begin{algorithmic}
\FOR{$t = 1$ to $T$}
  \STATE (In the contextual case) show $\A$ contexts $C_{t,1},\ldots,C_{t,K}$.
  \STATE Let $\A$ play action $i_t$
  \STATE Show $\A$ reward $D_{t,i_{t}}$.
\ENDFOR
\STATE Output $(i_1,\ldots,i_T)$.
\end{algorithmic}
\label{alg:interact}
\end{algorithm}

We denote the subset of the reward tableau $D$ corresponding to rewards that would have been revealed to a bandit algorithm $\A$ given action history $\Lambda_t^{\A}$, by $\Lambda_t^{\A}(D)$. Concretely if $\Lambda_t^{\A} = (i_1,\ldots,i_{t-1})$ then  $\Lambda_t^{\A}(D)= \{(i_\ell,y_{i_\ell,\ell})\}_{\ell =1}^{t-1}$. Given a selection function $f_t$ and an action history $\Lambda_t^{\A}$, define the query $q_{\Lambda_t^{\A}}$ as  $q_{\Lambda_t^{\A}}(D) = f_t({\Lambda_t^{\A}}(D))$.

We now define Algorithms \textbf{Bandit} and \textbf{InteractQuery}. \textbf{Bandit} is a standard contextual bandit algorithm defined by selection functions $f_t$,  and \textbf{InteractQuery} is the \textbf{Interact} routine that draws the rewards in advance, and at time $t$ selects action $i_t$ as the result of query $q_{\Lambda_t^{\A}}$. With the above definitions in hand, it is straightforward to show that the two Algorithms are equivalent, in that they induce the same joint distribution on their outputs. In both algorithms for convenience we assume we are in the linear contextual setting, and we write $\eta_{i_t}$ to denote the i.i.d. error distributions of the rewards, conditional on the contexts.

\begin{algorithm}
\textbf{Bandit}\newline
\textbf{Inputs: $T, k, \{\xit\}, \{\theta_i\}, f_t, \Lambda_0 = \emptyset$}
\begin{algorithmic}[1]
\FOR{$t = 1,\ldots ,T$:}
  \STATE Let $i_{t} = f_t(\Lambda_{t-1})$
  \STATE Draw $y_{i_t,t} \sim x_{i_t,t} \cdot \theta_{i_t} + \eta_{i_t}$
  \STATE Update $\Lambda_t = \Lambda_{t-1} \cup (i_t, y_{i_t,t})$
\ENDFOR
 \STATE \textbf{Return: $\Lambda_{T}$}
\end{algorithmic}
\end{algorithm}

 \begin{algorithm}
 \textbf{InteractQuery}\newline
  \textbf{Inputs: $T, k, D:D_{it}=\theta_i \cdot \xit+\eta_{it}, f_t$}
    \begin{algorithmic}[1]
     \FOR{ $t = 1,\ldots ,T$:}
     	\STATE Let $q_{t}= q_{\Lambda_{t-1}^{\A}}$
	\STATE Let $i_t= q_t(D)$
	\STATE Update $\Lambda_t^{\A} = \Lambda_{t-1}^{\A} \cup i_t$
   \ENDFOR
  \STATE \textbf{Return: $\Lambda_T^{\A}$}
\end{algorithmic}
 \end{algorithm}

\begin{claim}
\label{duh}
Let $P_{1,t}$ be the joint distribution induced by Algorithm \textbf{Bandit} on $\Lambda_t$ at time $t$, and let $P_{2,t}$ be the joint distribution induced by Algorithm \text{\textbf{InteractQuery}} on $\Lambda_t = \Lambda_t^{\A}(D)$. Then $\forall t \; P_{1,t} = P_{2,t}$.
\end{claim}

The upshot of this equivalence is that we can import existing results that hold in the setting in which the dataset is fixed, and queries are adaptively chosen. There are a large collection of results of this form that apply when the queries are differentially private \cite{preserve,agstab, maxinformation2} which apply directly to our setting. In the next section we formally define differential privacy in the simple stochastic and contextual bandit setting, and leave the description of the more general transfer theorems to Section~\ref{mutual}.

\subsection{Differential Privacy}
We will be interested in algorithms that are differentially private. In the simple stochastic bandit setting, we will require differential privacy with respect to the rewards. In the contextual bandit setting, we will also require differential privacy with respect to the rewards, but \emph{not necessarily} with respect to the contexts.

We now define the neighboring relation we need to define bandit differential privacy:
\begin{definition}
In the simple stochastic setting, two bandit tableau's $D, D'$ are \emph{reward neighbors} if they differ in at most a single row: i.e. if there exists an index $\ell$ such that for all $t \neq \ell$, $D_{t} = D'_{t}$.
\end{definition}
In the contextual setting, two bandit tableau's $(D, C), (D', C')$ are \emph{reward neighbors} if $C = C'$ and $D$ and $D'$ differ in at most a single row: i.e. if there exists an index $\ell$ such that for all $t \neq \ell$, $D_{t} = D'_{t}$.

Note that changing a \emph{context} does not result in a neighboring tableau: this neighboring relation will correspond to privacy for the rewards, but not for the contexts.

\begin{remark}
Note that we could have equivalently defined reward neighbors to be tableaus that differ in only a single entry, rather than in an entire row. The distinction is unimportant in a bandit setting, because a bandit algorithm will be able to observe only a single entry in any particular row.
\end{remark}

\begin{definition}
A bandit algorithm $\A$ is $(\epsilon,\delta)$ reward differentially private if for every time horizon $T$ and every pair of bandit tableau $\mathrm{Tab}, \mathrm{Tab}'$ that are reward neighbors, and every subset $S \subseteq [K]^T$:
$$\Pr{\mathbf{Interact}(T, \A, \mathrm{Tab}) \in S} \leq e^\epsilon \Pr{\mathbf{Interact}(T, \A, \mathrm{Tab}') \in S} + \delta$$
If $\delta = 0$, we say that $\A$ is $\epsilon$-differentially private.
\end{definition}

\subsection{The Binary Mechanism}
For many interesting stochastic bandit algorithms $\A$ (UCB, Thompson-sampling, $\epsilon$-greedy) the selection functions $(f_t)_{t \in [T]}$ are randomized functions of the history of sample means at each time step for each arm.  It will therefore be useful to have notation to refer to these means. We write $N_i^T$ to represent the number of times arm $i$ is pulled through round $T$: $N_i^T = \sum_{t' = 1}^{T}\mathbbm{1}_{\{f_{t'}(\Lambda_{t'}) = i\}}$. Note that before the history has been fixed, this is a random variable. In the simple stochastic setting, We write $\hat{Y}_i^{T}$ to denote the sample mean at arm $i$ at time $T: \hat{Y}_i^{T} = \frac{1}{N_i^{T}}\sum_{j=1}^{N_i^{T}}y_{i, t_j}$, where $t_j$ is the time $t$ that arm $i$ is pulled for the $j^{th}$ time. Then we can write the current set of sample means sequences for all $K$ arms at time $T$ as $(\hat{Y}_i^{t})_{i \in [K], t \in [T]}$. Since differential privacy is preserved under post-processing and composition, we observe that to obtain a private version $\A_{priv}$ of any of these standard algorithms, an obvious method would be to estimate $(\hat{Y}_i^{t})_{i \in [K]}$ privately at each round, and then to plug these private estimates into the selection functions $f_t$.

The Binary mechanism \cite{continual, DworkBin} is an online algorithm that continually releases an estimate of a running sum $\sum_{i = 1}^{t}y_i$ as each $y_i$ arrives one at a time, while preserving $\epsilon$-differential privacy of the entire sequence $(y_i)_{i=1}^{T}$, and guaranteeing worst case error that scales only with $\log(T)$. It does this by using a tree-based aggregation scheme that computes partial sums online using the Laplace mechanism, which are then combined to produce estimates for each sample mean $\hat{Y}_i^{t}$. Since the scheme operates via the Laplace mechanism, it extends immediately to the setting when each $y_i$ is a vector with bounded $l_1$ norm. In our private algorithms we actually use a modified version of the binary mechanism due to \cite{continual} called the hybrid mechanism, which operates without a fixed time horizon $T$. For the rest of the paper we denote the noise added to the $t^{th}$ partial sum by the hybrid mechanism, either in vector or scalar form, as $\eta \sim \text{Hybrid}(t,\epsilon)$. \\

\begin{theorem}[Corollary 4.5 in \cite{continual}]
Let $y_1, \ldots y_T \in [0,1]$. The hybrid mechanism
produces sample means $\tilde{Y}^{t} = \frac{1}{t}(\sum_{i=1}^{t}y_i + \eta_t),$
where $\eta_t \sim \text{Hybrid}(t,\epsilon)$, such that the following hold:
\begin{enumerate}
\item The sequence $(\tilde{Y}^{t})_{t \in [T]}$ is $\epsilon$-differentially private in $(y_1, \ldots y_T)$.
\item With probability $1-\delta$ :
\begin{equation}
\label{binary}
\sup_{t \in [T]}|\tilde{Y}_{i}^{T}-\hat{Y}_{i}^{T}| \leq \frac{\log^{*}(\log t)}{\epsilon t}\log t^{1.5}\text{Ln}(\log \log t)\log{\frac{1}{\delta}},
\end{equation}
where $\log^*$ denotes the binary iterated logarithm, and $\text{Ln}$ is the function defined as $\text{Ln}(n) = \prod_{r = 0}^{\log^*(n)}\log^{(r)}n$ in \cite{continual}.
\end{enumerate}
\end{theorem}
For the rest of the paper, we denote the RHS of (\ref{binary}) as $\tilde{O}(\frac{1}{t\epsilon}\log^{1.5}t\log \frac{1}{\delta})$, hiding the messier sub-logarithmic terms.

\section{Privacy Reduces Bias in Stochastic Bandit Problems}
\label{regret}
We begin by showing that differentially private algorithms that operate in the stochastic bandit setting compute empirical means for their arms that are nearly unbiased.
Together with known differentially private algorithms for stochastic bandit problems, the result is an algorithm that obtains a nearly optimal (worst-case) regret guarantee while also guaranteeing that the collected data is nearly unbiased. We could (and do) obtain these results by combining the reduction to answering adaptively selected queries given by Theorem~\ref{duh} with the standard generalization theorems in adaptive data analysis (e.g. Corollary~\ref{pvalcor} in its most general form), but we first prove these de-biasing results from first principles to build intuition.


\begin{theorem}
\label{notransfer}
Let $\A$ be an $(\epsilon,\delta)$-differentially private algorithm in the stochastic bandit setting. Then, for all $i \in [K]$, and all $t$, we have:


$$\left|\E{\hat{Y}_i^{t}-\mu_i}\right| \leq (e^\epsilon - 1+ T\delta)\mu_i$$
\end{theorem}
\begin{remark}
Note that since $\mu_i \in [0,1]$, and for $\epsilon \ll 1$, $e^\epsilon \approx 1+\epsilon$, this theorem bounds the bias by roughly $\epsilon+T\delta$. Often, we will have $\delta = 0$ and so the bias will be bounded by roughly $\epsilon$.
\end{remark}

\begin{proof}
First we fix some notation. Fix any time horizon $T$, and let $(f_t)_{t \in [T]}$ be the sequence of selection functions induced by algorithm $\A$.  Let $\mathbbm{1}_{\{f_t(\Lambda_t) = i\}}$ be the indicator for the event that arm $i$ is pulled at time $t$. We can write the random variable representing the sample mean of arm $i$ at time $T$ as
$$ \hat{Y}_i^{T} = \sum_{t = 1}^{T}\frac{\mathbbm{1}_{\{f_t(\Lambda_t) = i\}}}{\sum_{t' = 1}^{T}\mathbbm{1}_{\{f_{t'}(\Lambda_{t'}) = i\}}}y_{it}$$
where we recall that $y_{i,t}$ is the random variable representing the reward for arm $i$ at time $t$. Note that the numerator ($f_t(\Lambda_t) = i$) is by definition independent of $y_{i,t}$, but the denominator ($\sum_{t' = 1}^{T}\mathbbm{1}_{\{f_{t'}(\Lambda_{t'}) = i\}}$) is not, because for $t' > t$ $\Lambda_{t'}$ depends on $y_{i,t}$. It is this dependence that leads to bias in adaptive data gathering procedures, and that we must argue is mitigated by differential privacy.

We recall that the random variable $N_i^T$ represents the number of times arm $i$ is pulled through round $T$: $N_i^T = \sum_{t' = 1}^{T}\mathbbm{1}_{\{f_{t'}(\Lambda_{t'}) = i\}}$. Using this notation, we write the sample mean of arm $i$ at time $T$, as:
$$\hat{Y}_i^{T} = \sum_{t = 1}^{T}\frac{\mathbbm{1}_{\{f_t(\Lambda_t) = i\}}}{N_i^{T}}\cdot y_{it} $$
We can then calculate:
\begin{eqnarray*}
\bE[\hat{Y}_i^{t}] &=& \sum_{t = 1}^{T}\bE[\frac{\mathbbm{1}_{\{f_t(\Lambda_t) = i\}}}{N_i^{T}}y_{it}] \\
&=& \sum_{t=1}^T \bE_{y_{it} \sim P_i}[y_{it}\cdot \bE_{\A}[\frac{\mathbbm{1}_{\{f_t(\Lambda_t) = i\}}}{N_i^{T}}|y_{it}]]
\end{eqnarray*}
where the first equality follows by the linearity of expectation, and the second follows by the law of iterated expectation.

Our goal is to show that the conditioning in the inner expectation does not substantially change the value of the expectation. Specifically, we want to show that all $t$, and any value $y_{it}$, we have $$\bE[\frac{\mathbbm{1}_{\{f_t(\Lambda_t) = i\}}}{N_i}|y_{it}] \geq e^{-\epsilon}\bE[\frac{\mathbbm{1}_{\{f_t(\Lambda_t) = i\}}}{N_i^{T}}]-\delta$$ If we can show this, then we will have

$$\bE[\hat{Y}_i^{T}] \geq (e^{-\epsilon} \sum_{t=1}^{T}\bE[\frac{\mathbbm{1}_{\{f_t(\Lambda_t) = i\}}}{N_i^{T}}]- T\delta)\cdot\mu_i$$$$ = (e^{-\epsilon} \bE[\frac{N_i^T}{N_i^T}]-T\delta)\cdot \mu_i = (e^{-\epsilon}-T\delta)\cdot \mu_i$$
which is what we want (The reverse inequality is symmetric).

This is what we now show to complete the proof. Observe that for all $t,i$, the quantity $\frac{\mathbbm{1}_{\{f_t(\Lambda_t) = i\}}}{N_i}$ can be derived as a post-processing of the sequence of choices $(f_1(\Lambda_1),\ldots,f_T(\Lambda_T))$, and is therefore differentially private in the observed reward sequence.  Observe also that the quantity $\frac{\mathbbm{1}_{\{f_t(\Lambda_t) = i\}}}{N_i^{T}}$ is bounded in $[0,1]$. Hence by Lemma~\ref{expectation} for any pair of values $y_{it},y'_{it}$, we have $\bE[\frac{\mathbbm{1}_{\{f_t(\Lambda_t) = i\}}}{N_i^{T}}|y_{it}] \geq e^{-\epsilon}\bE[\frac{\mathbbm{1}_{\{f_t(\Lambda_t) = i\}}}{N_i^{T}}|y'_{it}]-\delta.$ All that remains is to observe that there must exist some value $y'_{it}$ such that $\bE[\frac{\mathbbm{1}_{\{f_t(\Lambda_t) = i\}}}{N_i}|y'_{it}] \geq \bE[\frac{\mathbbm{1}_{\{f_t(\Lambda_t) = i\}}}{N_i}]$. (Otherwise, this would contradict $\bE_{y'_{it} \sim P_i}[\bE[\frac{\mathbbm{1}_{\{f_t(\Lambda_t) = i\}}}{N_i}|y'_{it}]] = \bE[\frac{\mathbbm{1}_{\{f_t(\Lambda_t) = i\}}}{N_i^{T}}]$). Fixing any such $y'_{it}$ implies that for all $y_{it}$ $$\bE[\frac{\mathbbm{1}_{\{f_t(\Lambda_t) = i\}}}{N_i}|y_{it}] \geq e^{-\epsilon}\bE[\frac{\mathbbm{1}_{\{f_t(\Lambda_t) = i\}}}{N_i^{T}} | y'_{i,t}]-\delta$$$$ \geq e^{-\epsilon}\bE[\frac{\mathbbm{1}_{\{f_t(\Lambda_t) = i\}}}{N_i^{T}}]-\delta$$
 as desired. The upper bound on the bias follows symmetrically from Lemma~\ref{expectation}.
\end{proof}
\subsection{A Private UCB Algorithm}
There are existing differentially private variants of the classic UCB algorithm (\cite{Auer02, Agarwal,LaiRob}), which give a nearly optimal tradeoff between privacy and regret \cite{Mishra2,Tossou17, Tossou16}. 
For completeness, we give a simple version of a private UCB algorithm in the Appendix which we use in our experiments. Here, we simply quote the relevant theorem, which is a consequence of a theorem in \cite{Tossou16}:


  \begin{theorem}{\cite{Tossou16}}
  \label{regretTossou}
  Let $\{\mu_i: i \in [K]\}$ be the means of the $k$-arms. Let $\mu^* = \max_k \mu_k$, and for each arm $k$ let $\Delta = \min_{\mu_k < \mu^*} \mu^*-\mu_k$. Then there is an $\epsilon$-differentially private algorithm that obtains expected regret bounded by:
\begin{multline}
  \sum_{k \in [K]: \mu_k < \mu*} \min\left(\max\left(B(\ln(B)+7), \frac{32}{\Delta_k}\log T\right) +  \left(\Delta_k + \frac{2\pi^2\Delta_k}{3}\right), \Delta_k N_k^T\right)
\end{multline}
 where $B = \frac{\sqrt{8}}{2\epsilon}\ln(4T^4)$.
Taking the worst case over instances (values $\Delta_k$) and recalling that $\sum_k N_k^T = T$, this implies expected regret bounded by:
$$O\left(\max\left(\frac{\ln T}{\epsilon}\cdot \left(\ln\ln(T) + \ln(1/\epsilon)\right), \sqrt{kT\log T}\right)\right)$$
  \end{theorem}

Thus, we can take $\epsilon$ to be as small as $\epsilon = O(\frac{\ln^{1.5} T}{\sqrt{k T}})$ while still having a regret bound of $O(\sqrt{kT\log T})$, which is nearly optimal in the worst case (over instances) \cite{banditlb}.

 Combining the above bound with Theorem~\ref{notransfer}, and letting $\epsilon = O(\frac{\ln^{1.5} T}{\sqrt{k T}})$, we have:
  \begin{corollary}
  \label{simplecor}
  There exists a simple stochastic bandit algorithm that simultaneously guarantees that the bias of the empirical average for each arm $i$ is bounded by $O(\mu_i \cdot \frac{\ln^{1.5} T }{\sqrt{kT}})$ and guarantees expected regret bounded by $O(\sqrt{kT\log T})$.
  \end{corollary}
Of course, other tradeoffs are possible using different values of $\epsilon$. For example, the algorithm of \cite{Tossou16} obtains sub-linear regret so long as $\epsilon = \omega(\frac{\ln^2 T}{T})$. Thus, it is possible to obtain non-trivial regret while guaranteeing that the bias of the empirical means remains as low as $\mathrm{polylog}(T)/T$.

\section{Privacy Reduces Bias in Linear Contextual Bandit Problems}
In this section, we extend Theorem~\ref{notransfer} to directly show that differential privacy controls a natural measure of ``bias'' in linear contextual bandit problems as well. We then design and analyze a new differentially private algorithm for the linear contextual bandit problem, based on the Lin-UCB algorithm \cite{LI10}. This will allow us to give an algorithm which simultaneously offers bias and regret bounds.

In the linear contextual bandit setting, we first need to \emph{define} what we mean by bias. Recall that rather than simply maintaining an empirical mean for each arm, in the linear contextual bandit case, the algorithm is maintaining an estimate $\theta_{i,t}$ a linear parameter vector $\theta_i$ for each arm. One tempting measure of bias in this case is: $||\theta_i-\E{\hat{\theta}_{it}}||_2$, but even in the non-adaptive setting if the design matrix at arm $i$ is not of full rank, the OLS estimator will not be unique. In this case, the attempted measure of bias is not even well defined. Instead, we note that even when the design matrix is not of full rank, the predicted values on the training set $\hat{y} = x_{i,t} \hat{\theta}_{i,t}$ are unique. As a result we define bias in the linear contextual bandit setting to be the bias of \emph{the predictions that the least squares estimator, trained on the gathered data, makes on the gathered data}. We note that if the data were not gathered adaptively, then this quantity would be 0. We choose this one for illustration; other natural measures of bias can be defined, and they can be bounded using the tools in section \ref{mutual}.  

 We write $\Lambda_{i,T}$ to denote the sequence of context/reward pairs for arm $i$ that a contextual bandit algorithm $\A$ has observed through time step $T$. Note that $|\Lambda_{i,T}| = N_i^T$. It will sometimes be convenient to separate out contexts and rewards: we will write $C_{i,T}$ to refer to just the sequence of contexts observed through time $T$, and $D_{i,T}$ to refer to just the corresponding sequence of rewards observed through time $T$. Note that once we fix $\Lambda_{i,T}$, $C_{i,T}$ and $D_{i,T}$ are determined, but fixing $C_{i,T}$ leaves $D_{i,T}$ a random variable. The randomness in $C_{i,T}$ is over which contexts from arm $i$ $\A$ has selected by round $T$, not over the actual contexts $\xit$ - these are fixed. Thus the following results will hold over a worst-case set of contexts, including when the contexts are drawn from an arbitrary distribution. We will denote the sequence of arms pulled by $\A$ up to time $T$ by $\Lambda_T^{\A}$. We note that $\Lambda_T^{\A}$ fixes $C_{i,T}$ independently of the observed rewards $D_{i,T}$, and so if $\A$ is differentially private in the observed rewards, the post-processing $C_{i,T}$ is as well. First, we define the least squares estimator:

\begin{definition}
Given a sequence of observations $\Lambda_{i,T}$, a least squares estimator $\hat \theta_{i}$ is any vector that satisfies:
$$\hat\theta_{i} \in \arg\min_{\theta} \sum_{(\xit, \yit) \in \Lambda_{i,T}} (\theta \cdot \xit - \yit)^2$$
\end{definition}

\begin{definition}[Bias]\label{biascontext}

Fix a time horizon $T$, a tableau of contexts, an arm $i$, and a contextual bandit algorithm $\A$. Let $\hat{\theta}_i$ be the least squares estimator trained on the set of observations $\Lambda_{i,T}$. Then the bias of arm $i$ is defined to be the maximum bias of the \emph{predictions} made by $\hat{\theta}_i$ on the contexts in $C_{i,T}$, over any worst case realization of $C_{i,T}$. The inner expectation is over $D_{i,T}$ since $\hat{\theta}_i$ depends on the rewards at arm $i$.
$$\bias(i,T) = \max_{C_{i,T}, \;\xit \in C_{i,T}}{ \left|\Ex{D_{i,T}}{(\hat{\theta}_{i}-\theta_i)\xit}\right|}$$
\end{definition}

 It then follows from an elementary application of differential privacy similar to that in the proof of Theorem~\ref{notransfer}, that if the algorithm $\A$ makes its arm selection decisions in a way that is differentially private in the observed sequences of rewards, the least squares estimators computed based on the observations of $\A$ have bounded bias as defined above. The proof is deferred to the Appendix.

\begin{theorem}\label{contextbias}
Let $\A$ be any linear contextual bandit algorithm whose selections are $\epsilon$-differentially private in the rewards. Fix a time horizon $T$, and let $\hat\theta_i$ be a least squares estimator computed on the set of observations $\Lambda_{i,T}$. Then for every arm $i \in [K]$ and any round $t$:
$$\bias(i,T) \leq e^\epsilon - 1$$
\end{theorem}


Below we outline a reward-private variant of the LinUCB algorithm \cite{CLRS11}, and state a corresponding regret bound. In combination with Theorem~\ref{contextbias} this will give an algorithm that yields a smooth tradeoff between regret and bias. This algorithm is similar to the private linear UCB algorithm presented in \cite{Mishra2}. The main difference compared to the algorithm in \cite{Mishra2} is that Theorem~\ref{contextbias} requires only reward privacy, whereas the algorithm from \cite{Mishra2} is designed to guarantee privacy of the contexts as well. The result is that we can add less noise, which also makes the regret analysis more  tractable --- none is given in \cite{Mishra2} --- and the regret bound better. Estimates of the linear function at each arm are based on the ridge regression estimator, which gives a lower bound on the singular values of the design matrix and hence an upper bound on the effect of the noise. As part of the regret analysis we use the self-normalized martingale inequality developed in \cite{ridge}; for details see the proof in the Appendix.

\begin{algorithm}[H]
  \caption{\bf{LinPriv: Reward-Private Linear UCB}}
 \label{linpriv}
 \begin{algorithmic}[1]
 		\STATE \textbf{Input}:  $T, K$ algo params $\lambda,\delta,$ privacy budget $\epsilon$
		 \FOR{$t = 1, \ldots, T$}
 		\FOR {$i = 1, \ldots K$}
		\STATE Let $X_{it}, Y_{it} =$ design matrix, observed payoffs vector at arm $i$ before round $t$
		\STATE Let $\hat{V}_{it} = (X_{it}X_{it} + \lambda\textbf{I})$
		\STATE Draw noise $\eta_{it} \sim \text{Hybrid}(N_i^{t}, \epsilon)$
		\STATE Define $\hat{\theta}_{it}= (\hat{V}_{it})^{-1}(X_{it}'Y_{it})$ \COMMENT{the regularized LS estimator}
		\STATE Let $\hat{\theta}_{it}^{priv} = (\hat{V}_{it})^{-1}(X_{it}'Y_{it} + \eta_{it})$ \COMMENT{the private regularized LS estimator}
        		\STATE Observe $\xit$
		\STATE Let $\hat{y}_{tk} = \langle \hat{\theta}_{it}^{priv}, \xit \rangle$
		\STATE Let $w_{it} = ||\xit||_{\hat{V}_t^{-1}}(\sqrt{2d\log(\frac{1 + t/\lambda}{\delta})} + \sqrt{\lambda})$ \COMMENT{width of CI around private estimator}
		\STATE Let $s_{it} = \tilde{O}(\frac{1}{\epsilon}\log^{1.5}N_i^{t}\log \frac{K}{\delta})$
		\STATE Let $\text{UCB}_i(t) =  \hat{y}_{tk}+ \frac{1}{\lambda}s_{it} + w_{it}$
		\ENDFOR
		\STATE Let $i_t = \text{argmax}_{i \in [K]}\text{UCB}_i(t)$
		\STATE Observe $y_{i_t}$
		\STATE Update: $Y_{i_tt} \to Y_{i_tt+1}, X_{i_tt} \to X_{i_t{t+1}}$
		\ENDFOR
 \end{algorithmic}
  \end{algorithm}
\begin{theorem}
\label{regretlinear}
Algorithm~\ref{linpriv} is $\epsilon$-reward differentially private and has regret:
$$
R(T) \leq \tilde{O}(d\sqrt{TK} + \sqrt{TKd\lambda} +
K\frac{1}{\sqrt{\lambda}}\frac{1}{\epsilon}\log^{1.5}(T/K)\log(K/\delta)\cdot2d\log(1 + T/Kd\lambda)), $$
with probability $1-\delta$.

\end{theorem}
The following corollary follows by setting $\lambda =1$ and setting $\epsilon$ to be as small as possible, without it becoming an asymptotically dominant term in the regret bound. We then apply Theorem~\ref{contextbias} to convert the privacy guarantee into a bias guarantee.
\begin{corollary}
Setting $\lambda = 1$ and  $\epsilon = O(\sqrt{\frac{K}{T}})$, Algorithm~\ref{linpriv} has regret:
$$R(T) = \tilde{O}(d\sqrt{TK})$$
with probability $1-\delta$, and for each arm $i$ satisfies
$$\bias(i,T) \leq e^{\epsilon}-1 = O\left(\sqrt{\frac{K}{T}}\right)$$
\end{corollary}
\begin{remark}
Readers familiar with the linear contextual bandit literature will remark that the optimal non-private regret bound in the realizable setting scales like $O(\sqrt{Td \log K})$ \cite{CLRS11}, as opposed to $O(d\sqrt{TK})$ above. This is an artifact of the fact that for ease of presentation we have analyzed a simpler LinUCB variant using techniques from \cite{ridge}, rather than the more complicated SupLinUCB algorithm of \cite{CLRS11}. It is not a consequence of using the binary mechanism to guarantee privacy -- it is likely the same technique would give a private variant of SupLinUCB with a tighter regret bound than the one given above. 
\end{remark}

\section{Max Information \& Arbitrary Hypothesis Tests}\label{mutual}
Up through this point, we have focused our attention on showing how the private collection of data mitigates the effect that adaptivity has on \emph{bias}, in both the stochastic and contextual bandit problems. In this section, we draw upon more powerful results from the adaptive data analysis literature to go substantially beyond bias: to correct the $p$-values of hypothesis tests applied to adaptively gathered data. These $p$-value corrections follow from the connection between differential privacy and a quantity called \textit{max information}, which controls the extent to which the dependence of  selected test on the dataset can distort the statistical validity of the test \citep{maxinformation1, maxinformation2}. We briefly define max information, state the connection to differential privacy, and illustrate how max information bounds can be used to perform adaptive analyses in the private data gathering framework.
\begin{definition}[Max-Information \cite{maxinformation1}.]
Let $X, Z$ be jointly distributed random variables over domain $(\mathcal{X}, \mathcal{Z})$. Let $X \otimes Z$ denote the random variable that draws independent copies of $X, Z$ according to their marginal distributions. The max-information between $X, Z$, denoted $I_{\infty}(X,Z)$, is defined:
$$
I_{\infty}(X,Z)= \log \sup_{\mathcal{O} \subset (\mathcal{X} \times \mathcal{Z})}\frac{\Pr{(X,Z) \in \mathcal{O}}}{\Pr{X \otimes Z \in \mathcal{O}}}
$$
Similarly, we define the $\beta$-approximate max information
$$
I_{\beta}(X,Z)= \log \sup_{\mathcal{O} \subset (\mathcal{X} \times \mathcal{Z}),\; \Pr{(X,Z) \in \mathcal{O}}> \beta}\frac{\Pr{(X,Z) \in \mathcal{O}}-\beta}{\Pr{X \otimes Z \in \mathcal{O}}}
$$
\end{definition}

Following \cite{maxinformation2}, define a test statistic $t: \D \to \mathbb{R}$, where $\D$ is the space of all datasets. For $D \in \D$, given an output $a = t(D)$, the $p$-value associated with the test $t$ on dataset $D$ is $p(a) = \Prob{D \sim \mathbb{P}_0}{t(D) \geq a}$, where $P_0$ is the null hypothesis distribution. Consider an algorithm $\mathcal{A}$, mapping a dataset to a test statistic.

\begin{definition}[Valid $p$-value Correction Function \cite{maxinformation2}.]
A function $\gamma:[0,1] \to [0,1]$ is a valid $p$-value correction function for $\mathcal{A}$ if the procedure:
\begin{enumerate}
\item Select a test statistic $t = \mathcal{A}(D)$
\item Reject the null hypothesis if $p(t(D)) \leq \gamma(\alpha)$
\end{enumerate}
has probability at most $\alpha$ of rejection, when $D \sim P_0$.
\end{definition}

Then the following theorem gives a valid $p$-value correction function when $(D, A(D))$ have bounded $\beta$-approximate max information.
\begin{theorem}[\cite{maxinformation2}.]\label{thm:maxinf}
Let $\mathcal{A}$ be a data-dependent algorithm for selecting a test statistics such that $I_{\beta}(X, \mathcal{A}(X)) \leq k$. Then the following function $\gamma$ is a valid $p$-value correction function for $\mathcal{A}$:
$$
\gamma(\alpha) = \max(\frac{\alpha-\beta}{2^{k}},0)
$$

\end{theorem}
Finally, we can connect max information to differential privacy, which allows us to leverage private algorithms to perform arbitrary valid statistical tests.
\begin{theorem}[Theorem 20 from \cite{maxinformation1}.]\label{thm:maxinf2}
Let $\mathcal{A}$ be an $\epsilon$-differentially private algorithm, let $P$ be an arbitrary product distribution over datasets of size $n$, and let $D \sim P$. Then for every $\beta > 0$:
$$I_{\beta}(D, \A(D)) \leq \log(e)(\epsilon^2n/2 + \epsilon\sqrt{n\log(2/\beta)/2})$$
\end{theorem}
\cite{maxinformation2} extend this theorem to algorithm satisfying $(\epsilon,\delta)$-differential privacy.

\begin{remark}
We note that a hypothesis of this theorem is that the data is drawn from a product distribution. In the contextual bandit setting, this corresponds to rows in the bandit tableau being drawn from a product distribution. This will be the case if contexts are drawn from a distribution at each round, and then rewards are generated as some fixed stochastic function of the contexts. Note that contexts (and even rewards) can be correlated with one another within a round, so long as they are selected independently across rounds. In contrast, the regret bound we prove allows the contexts to be selected by an adversary, but adversarially selected contexts would violate the independence assumption needed for Theorem \ref{thm:maxinf2}.
\end{remark}

We now formalize the process of running a hypothesis test against an adaptively collected dataset. A bandit algorithm $\A$ generates a history $\Lambda_T \in \mathcal{H}^{T}$. Let the reward portion of the gathered dataset be denoted by $D_\A$.
 We define an \textit{adaptive test statistic selector} as follows.
 \begin{definition}
 Fix the reward portion of a bandit tableau $D$ and bandit algorithm $\A$. An adaptive test statistic selector is a function $s$ from action histories to test statistics such that
 $s(\Lambda_T^{\A})$ is a real-valued function of the adaptively gathered dataset ${D_{\A}}$. \end{definition}

 Importantly, the selection of the test statistic $s(\Lambda_T^{\A})$ can depend on the sequence of arms pulled by $\A$ (and in the contextual setting, on all contexts observed), but not otherwise on the reward portion of the tableau $D$. For example, $t_\A = s(\Lambda_T^{\A})$ could be the $t$-statistic corresponding to the null hypothesis that the arm $i^*$ which was pulled the greatest number of times has mean $\mu$:
 \[t_\A(D_{\A}) = \frac{\sum_{t=1}^{N_{i^*}^{T}}y_{i^*t}-\mu}{\sqrt{N_{i^*}^{T}}}\]
By virtue of Theorems~\ref{thm:maxinf} and \ref{thm:maxinf2}, and our view of  adaptive data gathering as adaptively selected queries, we get the following corollary:
\begin{corollary}
\label{pvalcor}
Let $\A$ be an $\epsilon$ reward differentially private bandit algorithm, and let $s$ be an adaptive test statistic selector. Fix $\beta > 0$, and let $\gamma(\alpha) = \frac{\alpha}{2^{\log(e)(\epsilon^2T/2 + \epsilon\sqrt{T\log(2/\beta)/2})}},$ for $\alpha \in [0,1]$. Then for any adaptively selected statistic $t_\A = s(\Lambda_T^{\A})$, and any product distribution $P$ corresponding to the null hypothesis for $t_{\A}$
$$
\Prob{D \sim P, \A}{p(t_\A(D)) \leq \gamma(\alpha)} \leq \alpha
$$
\end{corollary}
 If we set $\epsilon = O(1/\sqrt{T})$ in Corollary~\ref{pvalcor}, then  $\gamma(\alpha) = O(\alpha)$-- i.e. a valid $p$-value correction that only scales $\alpha$ by a constant. For example, in the simple stochastic setting, we can recall corollary~\ref{simplecor} to obtain:
\begin{corollary}
\label{regretpvalue}
Setting $\epsilon = O(\frac{\ln^{1.5} T}{\sqrt{k T}})$ there exists a simple stochastic bandit algorithm that guarantees expected regret bounded by $ O(\sqrt{kT\log T})$, such that for \textit{any} adaptive test statistic $t$ evaluated on the collected data, there exists a valid $p$-value correction function $\gamma(\alpha) = O(\alpha)$.
\end{corollary}
Of course, our theorems allow us to smoothly trade off the severity of the $p$-value correction with the regret bound.


\section{Experiments}
We first validate our theoretical bounds on bias in the simple stochastic bandit setting. As expected the standard UCB algorithm underestimates the mean at each arm, while the private UCB algorithm of \cite{Mishra} obtains very low bias. While using the $\epsilon$ suggested by the theory in Corollary~\ref{regretpvalue} effectively reduces bias and achieves near optimal asymptotic regret, the resulting private algorithm only achieves non-trivial regret
for large $T$ due to large constants and logarithmic factors in our bounds. This motivates a heuristic choice of $\epsilon$ that provides no theoretical guarantees on bias reduction, but leads to regret that is
comparable to the non-private UCB algorithm. We find empirically that even with this large choice of $\epsilon$ we achieve an $8$ fold reduction in bias relative to UCB. This is consistent with the observation that our guarantees hold in the worst-case, and suggests that there is room for improvement in our theoretical bounds --- both improving constants in the worst-case bounds on bias and on regret, and for proving instance specific bounds. Finally, we show that in the linear contextual bandit setting collecting data adaptively with a linear UCB algorithm and then conducting $t$-tests for regression coefficients yields incorrect inference (absent a $p$-value correction). These findings confirm the necessity of our methods when drawing conclusions from adaptively gathered data.

\subsection{Stochastic Multi-Armed Bandit}

In our first stochastic bandit experiment we set $K = 20$ and $T = 500$. The $K$ arm means are equally spaced between $0$ and $1$ with gap $\Delta = .05$, with $\mu_0 =1$.  We run UCB and $\epsilon$-private UCB for $T$ rounds with $\epsilon = .05$,  and after each run compute the difference between the sample mean at each arm and the true mean. We repeat this process $10,000$ times, averaging to obtain high confidence estimates of the bias at each arm. The average absolute bias over all arms for private UCB was $.00176$, with the bias for every arm being statistically indistinguishable from $0$ (see Figures~\ref{figsub1} for confidence intervals) while the average absolute bias (over arms) for UCB was $.0698$, or over $40$ times higher. The most biased arm had a measured bias of roughly $0.14$, and except for the top $4$ arms, the bias of each arm was statistically significant.  It is worth noting that private UCB achieves bias significantly lower than the $\epsilon = .05$ guaranteed by the theory, indicating that the theoretical bounds on bias obtained from differential privacy are conservative.  Figures~\ref{figsub1}, \ref{figsub2} show the bias at each arm for private UCB vs. UCB, with $95\%$ confidence intervals around the bias at each arm. Not only is the bias for private UCB an order of magnitude smaller on average, it does not exhibit the systemic negative bias evident in Figure~\ref{figsub2}.

Noting that the observed reduction in bias for $\epsilon = .05$ exceeded that guaranteed by the theory, we run a second experiment with $K = 5, T = 100000, \Delta = .05,$ and $\epsilon = 400$, averaging results over $1000$ iterations.
Figure~\ref{regret} shows that private UCB achieves sub-linear regret comparable with UCB. While $\epsilon = 400$ provides no meaningful theoretical guarantee, the average absolute bias at each arm mean obtained by the private algorithm was $.0015$ (statistically indistinguishable from 0 at 95\% confidence for each arm), while the non-private UCB algorithm obtained average bias $.011$, $7.5$ times larger. The bias reduction for the arm with the smallest mean (for which the bias is the worst with the non private algorithm) was by more than a factor of 10.  Figures~\ref{fig2sub1},\ref{fig2sub2} show the bias at each arm for the private and non-private UCB algorithms together with 95\% confidence intervals; again we observe a negative skew in the bias for UCB, consistent with the theory in \cite{gather}.

\begin{figure}
\label{figbias}
\centering
\begin{subfigure}
    \centering
  \includegraphics[scale=.6]{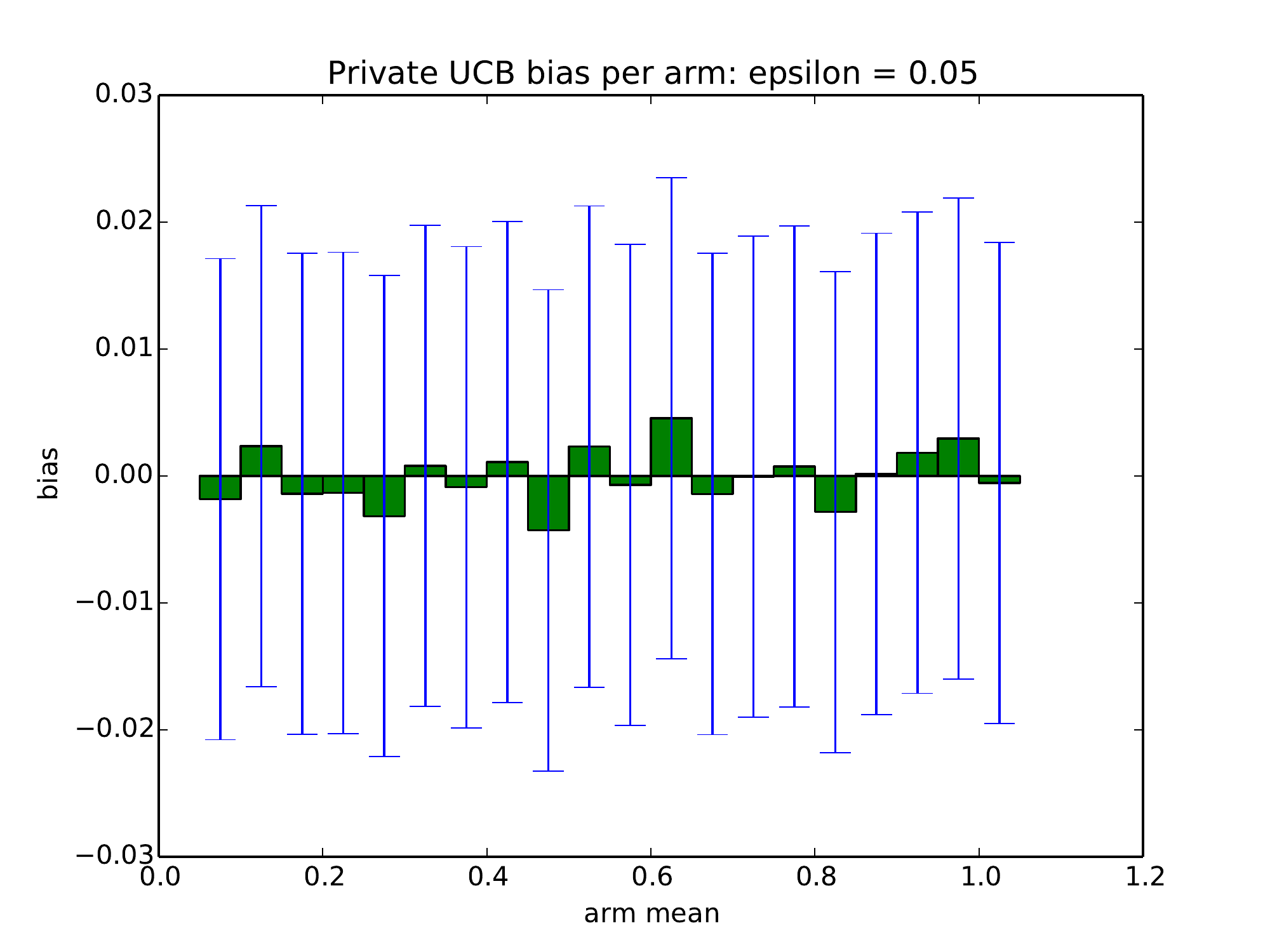}
  \caption{Private UCB Bias per Arm (experiment $1$)}
    \label{figsub1}
\end{subfigure}%
\begin{subfigure}  \centering
  \includegraphics[scale = .6]{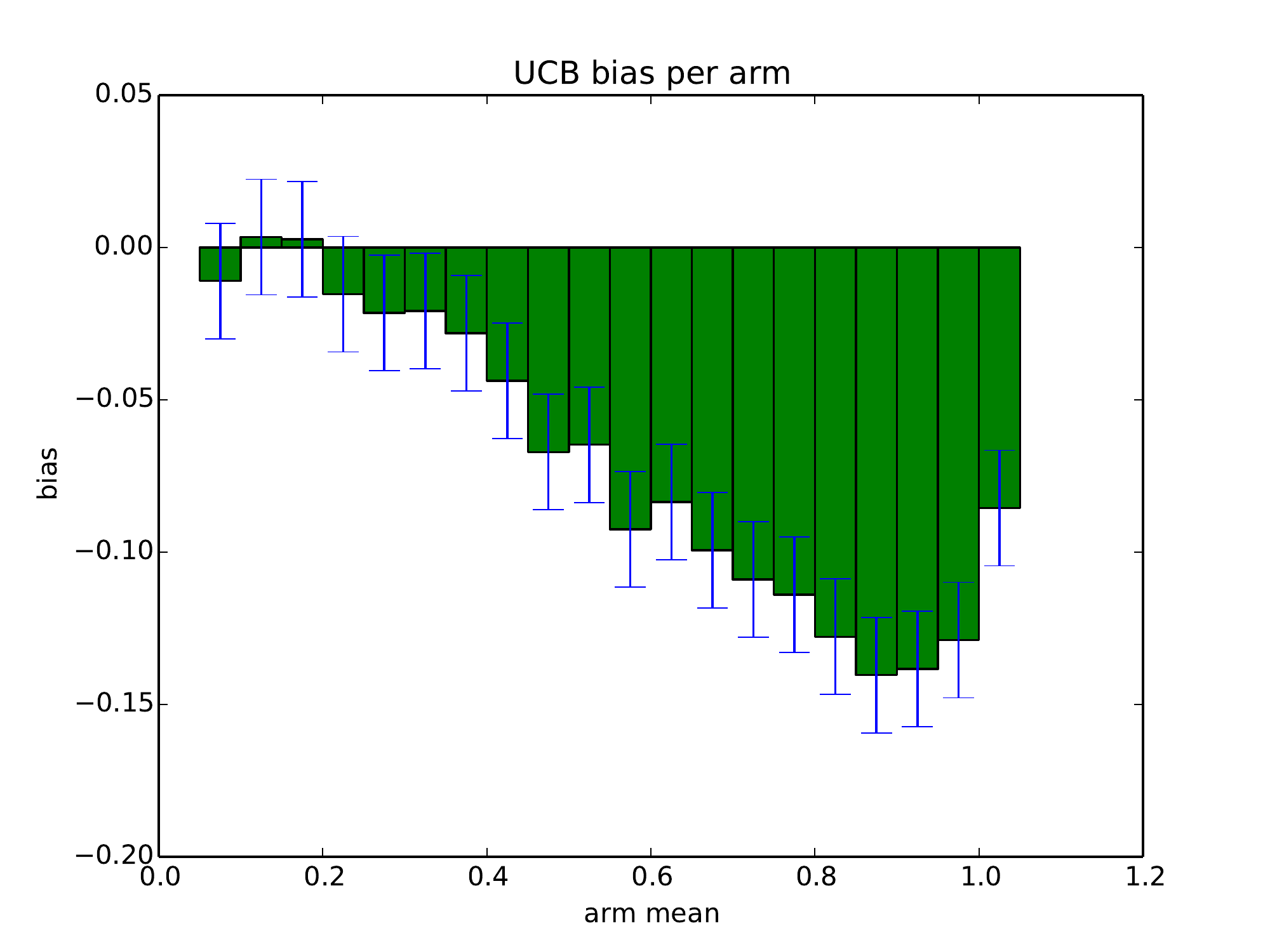}
\caption{UCB Bias per Arm (experiment $1$)}
  \label{figsub2}
\end{subfigure}
\end{figure}

\begin{figure}
\centering  \includegraphics[scale=.6]{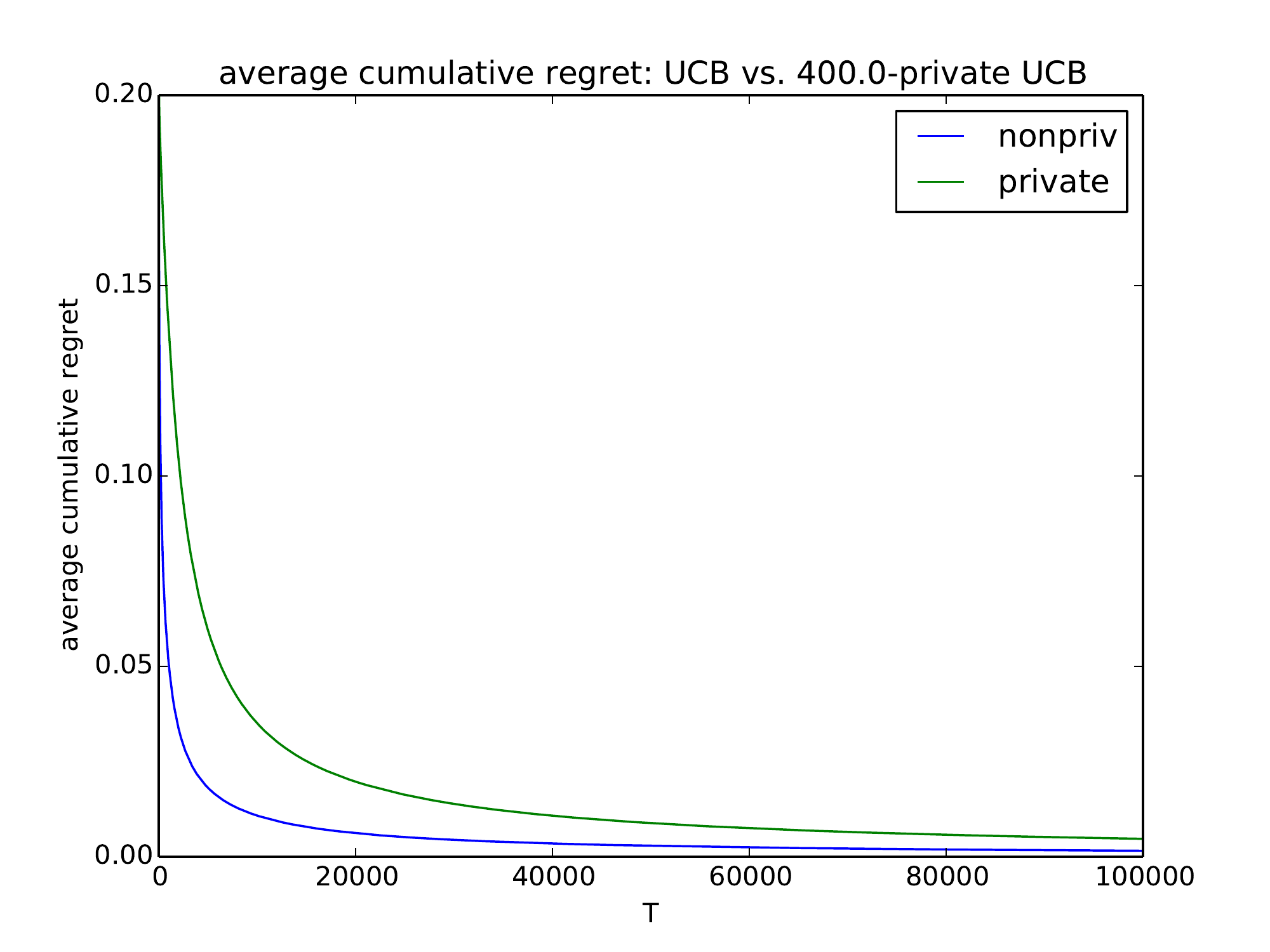}
  \caption{Average Regret: UCB vs. Private UCB}
  \label{regret}
\end{figure}

\begin{figure}
\label{bias2}
\centering
\begin{subfigure}  \centering
  \includegraphics[scale=.6]{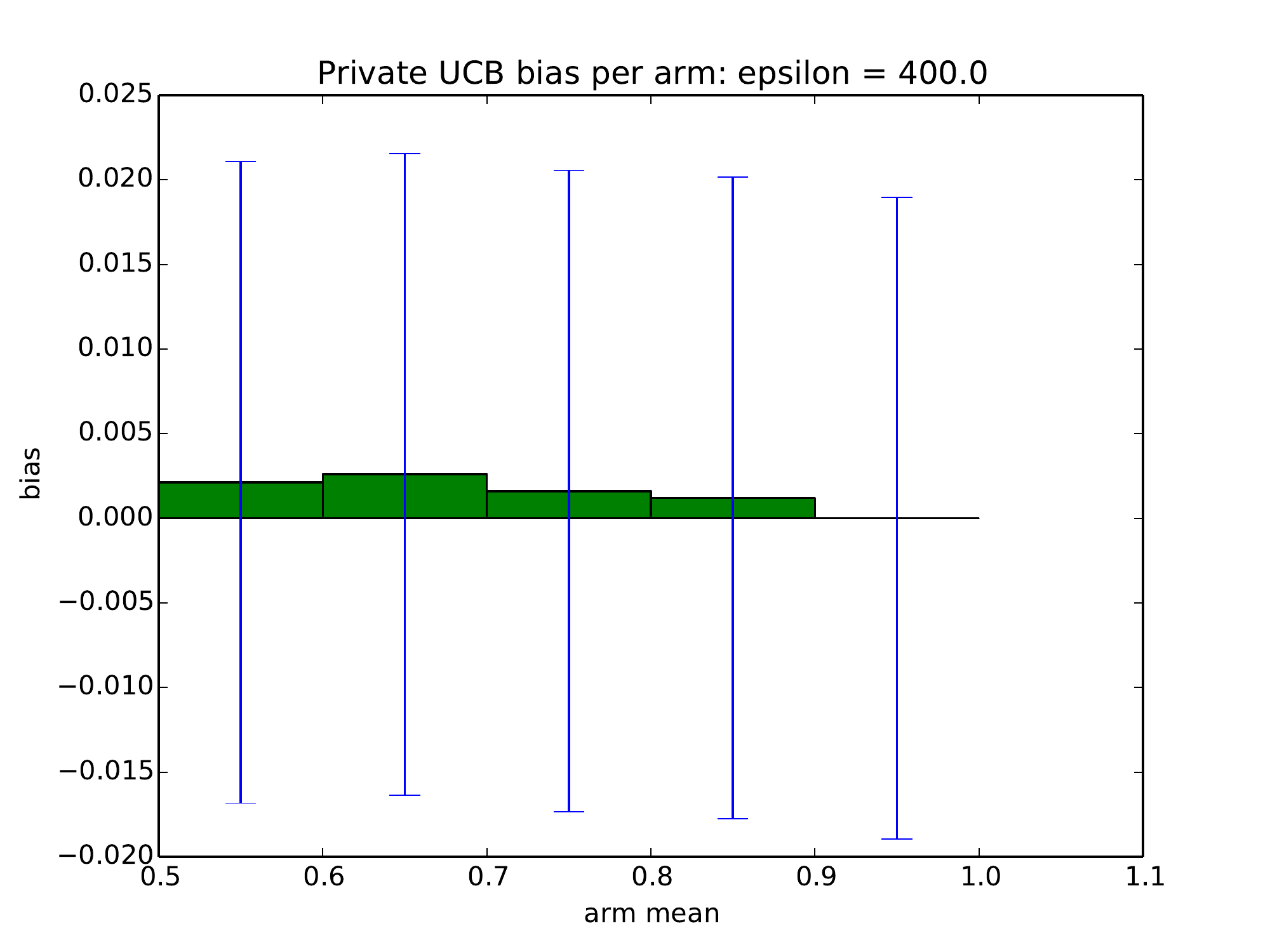}
  \caption{Private UCB Bias per Arm (experiment $2$)}
  \label{fig2sub1}
\end{subfigure}%
\begin{subfigure}  \centering
  \includegraphics[scale = .6]{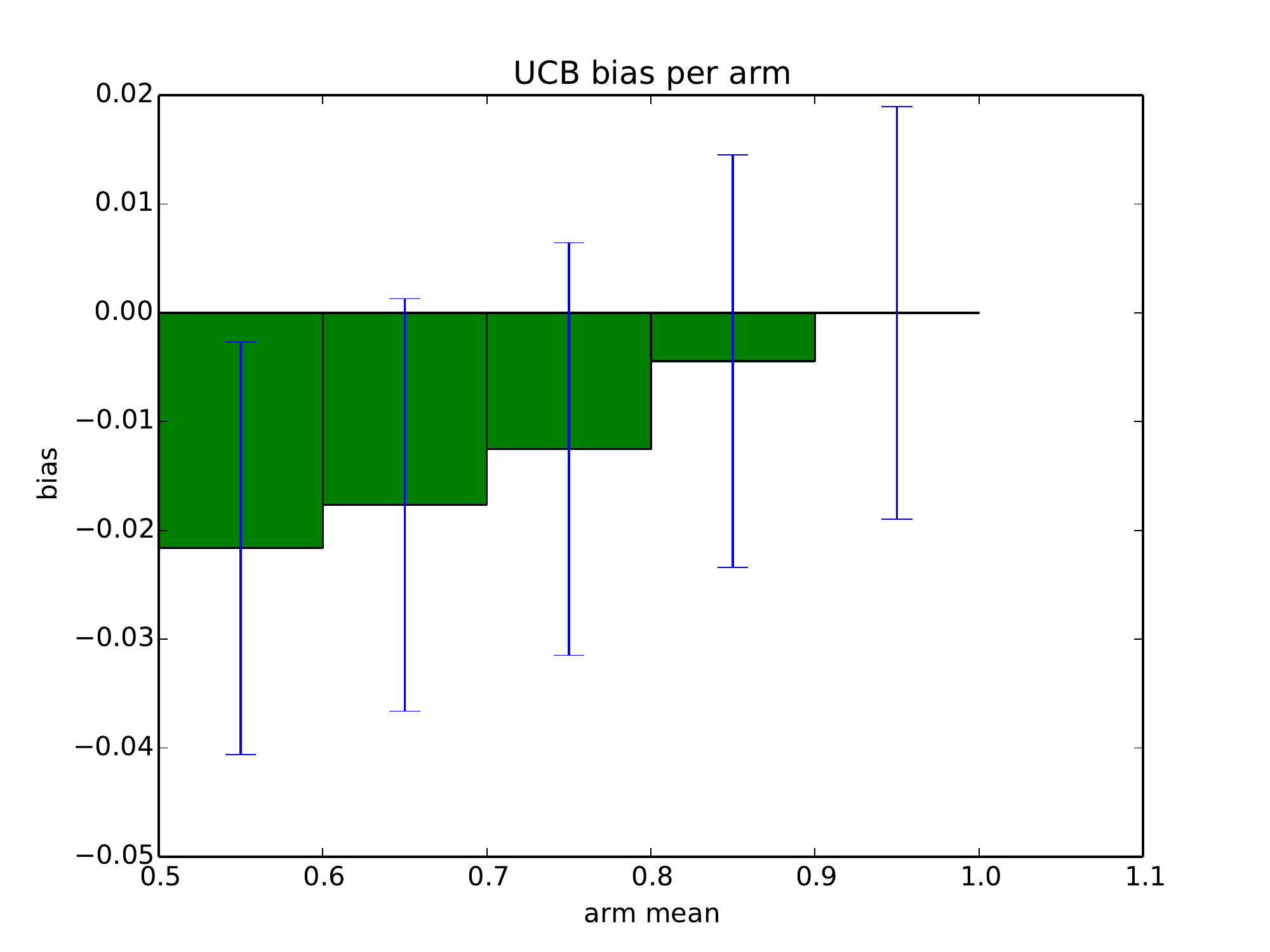}
\caption{UCB Bias per Arm (experiment $2$)}
  \label{fig2sub2}
\end{subfigure}
\end{figure}

\subsection{Linear Contextual Bandits}
 \cite{gather} prove and experimentally investigate the existence of negative bias at each arm in the simple stochastic bandit case. Our second experiment confirms that adaptivity leads to bias in the linear contextual bandit setting in the context of hypothesis testing -- and in particular can lead to false discovery in testing for non-zero regression coefficients. The set up is as follows: for $K = 5$ arms, we observe rewards $\yit \sim \mathcal{N}(\theta_i' \xit, 1)$, where $\theta_i, \xit \in \mathbb{R}^5,  ||\theta_i|| = ||\xit|| = 1$. For each arm $i$, we set $\theta_{i1} =0$. Subject to these constraints, we pick the $\theta$ parameters uniformly at random (once per run), and select the contexts $x$ uniformly at random (at each round). We run a linear UCB algorithm (OFUL \cite{ridge}) for $T = 500$ rounds, and identify the arm $i^*$ that has been selected most frequently. We then conduct a $z$-test for whether the first coordinate of $\theta_{i^*}$ is equal to $0$. By construction the null hypothesis $H_0: \theta_{i^*1} = 0$ of the experiment is true, and absent adaptivity, the $p$-value should be distributed uniformly at random. In particular, for any value of $\alpha$ the probability that the corresponding $p$-value is less than $\alpha$ is exactly $\alpha$. We record the observed $p$-value, and repeat the experiment $1000$ times, displaying the histogram of observed $p$-values in Figure~\ref{pvaluehistogram}. As expected, the adaptivity of the data gathering process leads the $p$-values to exhibit a strong downward skew. The dotted blue line  demarcates $\alpha = .05$. Rather than probability $.05$ of falsely rejecting the null hypothesis at 95\% confidence, we observe that $76\%$ of the observed $p$-values fall below the $.05$ threshold. This shows that a careful $p$-value correction in the style of Section~\ref{ada} is essential even for simple testing of regression coefficients, lest bias lead to false discovery.

 \begin{figure}
\label{fig:test}
\centering
\includegraphics[scale=.6]{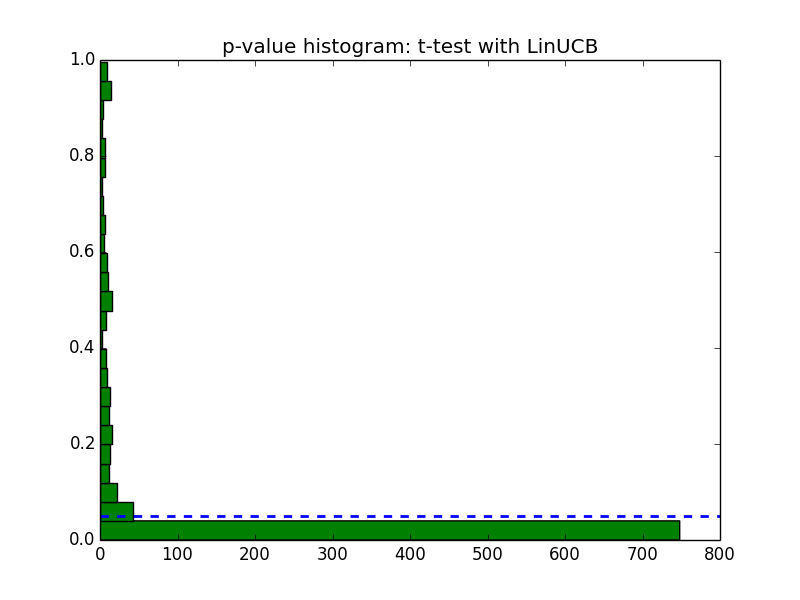}
\caption{Histogram of $p$-values for $z$-test under the null hypothesis. $K = d = 5, T = 500$.}
\label{pvaluehistogram}
\end{figure}
\pagebreak

\bibliographystyle{plainnat}
\bibliography{refs}

\appendix

\label{appendix}
%
%
%

\section{Differential Privacy Basics}
\label{privbasics}

We recall the standard definition of differential privacy, which can be defined over any neighboring relationship on data sets $D, D' \in \cX^*$. The standard relation says that $D, D'$ are neighbors (written as $D \sim D'$) if they differ in a single element.

\begin{definition}[Differential Privacy \cite{DMNS06}]
Fix $\eps \geq 0$. A randomized algorithm $A:\cX^*\rightarrow \mathcal{O}$ is $(\eps,\delta)$-differentially private if for every pair of neighboring data sets $D \sim D' \in \cX^*$, and for every event $S \subseteq \mathcal{O}$:
$$\Pr{A(D) \in S} \leq \exp(\eps)\Pr{A(D') \in S}+\delta.$$
\end{definition}
Differentially private computations enjoy two nice properties:
\begin{lemma}[Post Processing \cite{DMNS06}]
\label{post}
Let $A:\cX^*\rightarrow \mathcal{O}$ be any $(\eps,\delta)$-differentially private algorithm, and let $f:\mathcal{O}\rightarrow \mathcal{O'}$ be any (possibly randomized) algorithm. Then the algorithm $f \circ A: \cX^*\rightarrow \mathcal{O}'$ is also $(\eps,\delta)$-differentially private.
\end{lemma}
Post-processing implies that, for example, every \emph{decision} process based on the output of a differentially private algorithm is also differentially private.

\begin{theorem}[Composition \cite{DMNS06}]\label{composition}
Let $A_1:\cX^*\rightarrow \mathcal{O}$, $A_2:\cX^*\rightarrow \mathcal{O}'$ be algorithms that are $(\eps_1,\delta_1)$- and $(\eps_2,\delta_2)$-differentially private, respectively. Then the algorithm $A:\cX^*\rightarrow \mathcal{O}\times \mathcal{O'}$ defined as $A(x) = (A_1(x), A_2(x))$ is $(\eps_1+\eps_2),(\delta_1+\delta_2)$-differentially private.
\end{theorem}


%

\begin{definition}
Two random variables $X, Y$ defined over the same domain $R$ are $(\epsilon,\delta)$-close, written $X \approx_{\epsilon, \delta} Y$ , if for all $S \subseteq R$:
$$\Pr{X \in S} \leq e^\epsilon\Pr{Y \in S} + \delta$$
\end{definition}

Note that if $A$ is an $(\epsilon,\delta)$-differentially private algorithm, and $D, D'$ are neighboring datasets, then $A(D) \approx_{\epsilon,\delta} A(D')$. We make use of a simple lemma:

\begin{lemma}[Folklore, but see e.g. \cite{preserve}]\label{expectation}
Let $X, Y$ be distributions such that $X \approx_{\epsilon, \delta} Y$ and let $f: \mathcal{Y} \to [0,1]$ be a real-valued function on the outcome space. Then $\E{f(X)}\geq \exp(-\epsilon)\E{f(Y)} + \delta$
\end{lemma}

\section{Useful Concentration Inequalities}

\begin{lemma}[Hoeffding Bound (See e.g. \cite{com})]\label{hoeffding}
Let $X_1, \ldots X_n$ be independent random variables bounded by the interval $[0,1]: 0 \leq X_i \leq 1$. Then for $t > 0$,
$\Pr{|\bar{X}-\E{\bar{X}}| \geq t} \leq 2e^{-2nt^2}$\end{lemma}

\section{A Private UCB algorithm}
For completeness, we reproduce a version of the private UCB algorithm of \cite{Mishra} which we use in our experiments. See algorithm \ref{privucb}.

\begin{algorithm}[h]
  \caption{\bf{PrivUCB: Private Upper Confidence Bound \cite{Mishra}}}
 \label{privucb}
 \begin{algorithmic}
 \STATE \textbf{Input}: $T, K, \{P_i\}, \delta, \epsilon = \sqrt{K/T}$
 \FOR{$t = 1, \ldots, T$:}
 	\FOR {$k = 1, \ldots K$:}
	\STATE Draw $\eta_{N_i^{t}}\sim \text{Hybrid}(N_i^{t}, \frac{\epsilon}{K})$
	\STATE Confidence relaxation: $\gamma_t =  \frac{K \log^2 T \log(KT\log T/\delta)}{\epsilon}$
		\STATE Let $\text{UCB}_i(t) =  \eta_{N_i^{t}} + \hat{X}_{i}^{t} + \sqrt{\frac{2\log(t/\delta)}{N_i^{t}}} + \frac{\gamma_t}{N_i^{t}}$
	\ENDFOR
\STATE Let $i_t = \text{argmax}_{i \in [K]}\text{UCB}_i(t)$
\STATE Draw $x_t \sim P_{i_t}$
\STATE Update: $N_i^t \to N_i^{t} + 1, \hat{X}_i^{t} \to \hat{X}_i^{t+1}$
\ENDFOR
 \end{algorithmic}
  \end{algorithm}

\section{Missing Proofs}
\begin{proof}[Proof of Theorem~\ref{contextbias}]
Fix any $x_{ik} \in C_{i,T}$. We write $\dag$ to denote the matrix inverse in the case it exists, or else the pseudo-inverse if not. We first expand ${\hat{\theta}_{i}}'x_{ik}$:
$${\hat{\theta}_i}'x_{ik} = x_{ik}'(\sum_{x_{i,\ell} \in C_{i,T}}x_{i,\ell} x_{i,\ell}')^{\dag}\sum_{x_{i,\ell} \in C_{i,T}}x_{i,\ell}'y_{i,\ell}= 
x_{ik}'(\sum_{t=1}^{T}\xit \xit'\mathbbm{1}_{\Lambda_T^{\A}(t) = i})^{\dag}(\sum_{t=1}^{T}\xit\yit\mathbbm{1}_{\Lambda_T^{\A}(t) = i}),
$$
where $\mathbbm{1}_{\Lambda_T^{\A}(t) = i}$ is the indicator that arm $i$ was pulled at round $t$. Then we take the conditional expectation of $\hat{\theta}'x_{ik}$, conditioned on $\Lambda_T^{\A}$. Note that once we condition, $(\sum_{t=1}^{T}\xit \xit'\mathbbm{1}_{\Lambda_T^{\A}(t) = i})^{\dag}$ is just a fixed matrix, and so linearity of expectation will allow us to propagate through to the outer term:
$$
\Ex{D_{i,T}}{\hat{\theta}_i'x_{ik}|\Lambda_T^{\A}} = 
x_{ik}'(\sum_{t=1}^{T}\xit \xit'\mathbbm{1}_{\Lambda_T^{\A}(t) = i})^{\dag}(\sum_{t=1}^{T}\xit\mathbbm{1}_{\Lambda_T^{\A}(t)}\Ex{D_{i,T}}{\yit|\Lambda_T^{\A}})
$$
Note that we condition on $\Lambda_T^{\A}$ which is an $\epsilon$-differentially private function of the rewards $\yit$, and that $\yit \in [0,1]$. Hence by Lemma $3$, just as in the proof of Theorem~\ref{notransfer}, we have that $\Ex{D_{i,T}}{\yit|\Lambda_T^{\A}} \leq e^{\epsilon}\Ex{D_{i,T}}{\yit} = e^{\epsilon}\xit\cdot \theta_i$. Substituting into the above gives:
$$
\Ex{D_{i,T}}{\hat{\theta}_i'x_{ik}|\Lambda_T^{\A}} \leq e^{\epsilon}x_{ik}'(\sum_{t=1}^{T}\xit \xit'\mathbbm{1}_{\Lambda_T^{\A}(t) = i})^{\dag}(\sum_{t=1}^{T}\xit\mathbbm{1}_{\Lambda_T^{\A}(t)}\xit'\theta_i) $$
$$=e^{\epsilon}x_{ik}'(\sum_{t=1}^{T}\xit \xit'\mathbbm{1}_{\Lambda_T^{\A}(t) = i})^{\dag}(\sum_{t=1}^{T}\xit \xit'\mathbbm{1}_{\Lambda_T^{\A}(t) = i})\theta_i = e^{\epsilon}x_{ik}'\theta_i,
$$
where the last equality follows immediately when $\sum_{t=1}^{T}\xit \xit'\mathbbm{1}_{\Lambda_T^{\A}(t) = i}$ is full-rank, and follows from properties of the pseudo-inverse even if it is not. But then we've shown that $\Ex{D_{i,T}}{\hat{\theta}_i'x_{ik}-\theta_{i}'x_{ik}|\Lambda_T^{\A}} \leq (e^{\epsilon}-1)\theta_i'x_{ik} \leq e^{\epsilon}-1$, since by assumption $|\theta_i'x_{ik}| \leq 1$. Since this expectation holds conditionally on $\Lambda_{T}^{\A}$, we can integrate out $\Lambda_T^{\A}$ to obtain:
$$
\Ex{D_{i,T}}{\hat{\theta}_i'x_{ik}-\theta_{i}'x_{ik}} \leq e^{\epsilon}-1
$$
 The lower bound $-\theta_{i}'x_{ik}-\Ex{D_{i,T}}{\hat{\theta}_i'x_{ik}} \geq 1-e^{-\epsilon}$ follows from the reverse direction of Lemma $3$. Since this holds for arbitrary $C_{i,T}$ and $x_{ik} \in C_{i,T}$ we are done.
\end{proof}

\begin{proof}[Proof of Theorem~\ref{regretlinear}] The reward-privacy claim follows immediately from the privacy of the hybrid mechanism \cite{continual} and the post-processing property of differential privacy  (Lemma~\ref{post}). Here we prove the regret bound. 
We first show that the confidence intervals given by $\hat{y}_{tk} \pm (\frac{1}{\lambda}s_{it} + w_{it})$ are valid $\forall i,t$ with probability $1-\delta$.
Then since we always play the action with the highest upper confidence bound, with high probability we can bound our regret at time $T$ by the sum of the widths of the confidence intervals of the chosen actions at each time step. \\

We know from \cite{ridge} that $\forall i,T, \Pr{\langle \theta_{it}, \xit \rangle \in [\langle\hat{\theta}_{it},\xit \rangle \pm w_{it}]} \geq 1- \frac{\delta}{2}$.
By construction,
\begin{equation}\label{conf}
|\inner{\hat{\theta}_{it}^{priv}}{\xit} - \inner{\hat{\theta}_{it}}{\xit}|  \leq |\xit'\hat{V}_{it}^{-1}\eta_{it}| \leq ||\xit||_{\hat{V}_{it}^{-1}} ||\eta_{it}||_{\hat{V}_{it}^{-1}}, \end{equation}
where the second inequality follows from applying the Cauchy-Schwarz inequality with respsect to the matrix inner product $\langle \cdot, \cdot \rangle_{\hat{V}_{it}^{-1}}$. We also have that $||\eta_{it}||_{\hat{V}_{it}^{-1}} \leq 1/\sqrt{\lambda}||\eta_{it}||_2$, and by the utility theorem for the Hybrid mechanism \cite{continual}, with probability $1-\delta/2,\; \forall i,t, \; ||\eta_{it}||_2 \leq s_{it} = O(\frac{1}{\epsilon}\log^{1.5}T\log(K/\delta))$.
Thus by triangle inequality and a union bound, with probability $1-\delta, \;\forall i,t$:
$$|\inner{\theta_{it}}{\xit}-\inner{\hat{\theta}_{it}^{priv}}{\xit}|  \leq  O(\frac{1}{\sqrt{\lambda}}\frac{1}{\epsilon}\log^{1.5}T\log(K/\delta)||\xit||_{\hat{V}_{it}^{-1}})+ w_{it},
$$
Let $R(T)$ denote the pseudo-regret at time $T$, and $R_i(T)$ denote the sum of the widths of the confidence intervals at arm $i$, over all times in which arm $i$ was pulled. Then with probability $1-\delta$:
$$ R(T) \leq \sum_{i} R_i(T) \leq  \sum_{i = 1}^{K} \bigg( \sum_{t=1}^{N_i^{T}}w_{i_t t}   + \frac{1}{\sqrt{\lambda}}\frac{1}{\epsilon}\log^{1.5}N_i^{T}\log(K/\delta)\bigg(K\sum_{i=1}^{N_i^{T}/K}||\xit||_{\hat{V}_{it}^{-1}}\bigg)\bigg)$$
The RHS is maximized at $N_i^{T} = \frac{T}{K}$ for all $i$, giving:
$$
R(T) \leq K\bigg( \sum_{t=1}^{T/K}w_{i_t t} + \frac{1}{\sqrt{\lambda}}\frac{1}{\epsilon}\log^{1.5}(T/K)\log(K/\delta) \bigg(K\sum_{i=1}^{T/K}||\xit||_{\hat{V}_{it}^{-1}}\bigg)\bigg)
$$

Reproducing the analysis of \cite{ridge}, made more explicit on page $13$ in the Appendix of \cite{neel} gives:
$$
 \sum_{i =1}^{T/K}w_{i_t t}  \leq \sqrt{2d \log(1 + \frac{T}{\lambda K d})} 
 (\sqrt{2dT/K \log(1/\delta + \frac{T}{K\lambda \delta})} + \sqrt{\frac{T}{K}\lambda})
$$
  The crux of their analysis is actually the bound $\sum_{t=1}^{n}||\xit||_{\hat{V}_{it}^{-1}} \leq 2d\log(1 + n/d\lambda) $, which holds for $\lambda \geq 1$. Letting $n = T/K$ bounds the second summation, giving that with probability $1-\delta$:

$$
R(T) = \tilde{O}(d\sqrt{TK} + \sqrt{TKd\lambda} \; + 
K\frac{1}{\sqrt{\lambda}}\frac{1}{\epsilon}\log^{1.5}(T/K)\log(K/\delta)\cdot2d\log(1 + T/Kd\lambda)), $$

 where $\tilde{O}$ hides logarithmic terms in $1/\lambda, 1/\delta, T, K, d$.
\end{proof}

\begin{proof}[Proof of Claim~\ref{duh}]
We first remark that by the principle of deferred randomness we can view Algorithm $3$ as first drawing the tableau $D \in ([0,1]^{K})^{T}$ and receiving $C \in ((\mathbb{R}^d)^K)^{T}$ up front, and then in step $4$ publishing $y_{i_t,t}$ rather than drawing a fresh $y_{i_t, t}$. Then, because for both Algorithm $2$ and Algorithm $3$ the tableau distributions are the same, it suffices to show that conditioning on $D$, the distributions induced on the action histories $\Lambda_t^{\A}$ are the same. For both algorithms, at round $t$, there is some distribution over the next arm pulled $i_t$. We can write the joint distribution over $\Lambda_{t+1}^{\A} = (i_1,\ldots,i_t)$ as:
$$
\Pr{i_1, \ldots i_t} = \prod_{k = 1}^{t}\Pr{i_k|i_{k-1}, \ldots i_1}
$$
For Algorithm $2$ $\Pr{i_k|i_{k-1}, \ldots i_1}$ is equal to $\Pr{f_k(\Lambda_k) = i_k}$. For Algorithm $3$ it is $\Pr{q_k(D) = i_k}$. But by definition $\Pr{q_k(D) = i_k} = \Pr{q_{\Lambda_{k}^{\A}}(D) = i_k} = \Pr{f_k(\Lambda_t^{\A}(D)) = i_k} = \Pr{f_k(\Lambda_k) = i_k}$, and so the joint distributions coincide.
\end{proof}

\end{document}